\documentclass[11pt]{article}
\usepackage[utf8]{inputenc}
\usepackage[margin=1in]{geometry}

\usepackage[round]{natbib}
\usepackage{amsfonts}
\usepackage{amsthm}

\usepackage{amsmath}
\usepackage{mathtools}
\usepackage{amssymb}
\usepackage{etoolbox}
\usepackage{amssymb}

\usepackage{verbatim}
\usepackage{tabto}
\usepackage{enumitem}

\usepackage{fancyhdr}
\usepackage{xspace}
\usepackage{xcolor}
\usepackage{xparse}

\usepackage[font=small]{caption}

\usepackage{booktabs}
\usepackage{array}
\usepackage{arydshln}
\usepackage[extdef=true]{delimset}

\usepackage{titlesec}
\usepackage{soul}
\usepackage[colorlinks,allcolors=blue]{hyperref}

\usepackage[capitalize]{cleveref}

\hypersetup{
           breaklinks=true,   % splits links across lines
}
\titleformat*{\paragraph}{\bfseries\itshape}

\usepackage{graphicx}
\usepackage{etoolbox}
\usepackage{xcolor}

\usepackage{float}
\usepackage{verbatim}
\usepackage{algorithm}
\usepackage[noend]{algpseudocode}

\renewcommand{\epsilon}{\varepsilon}
\newcommand{\E}{\mathbb{E}} % expectation\
\newcommand{\R}{\mathbb{R}} % expectation
\newcommand{\D}{\mathcal{D}}

\newcommand{\prob}{\delta}
\newcommand{\loss}{\ell}
\newcommand{\classfunc}{\phi}
\newcommand{\margin}{\gamma}

\newcommand{\classtail}{\mathcal{C}_{\classfunc,\beta}}
\newcommand{\normball}{B_{\epsilon}}

\usepackage{xcite}
\PassOptionsToPackage{unicode}{hyperref}
\crefname{claim}{Claim}{Claims}
\crefname{assumption}{Assumption}{Assumptions}

\usepackage{booktabs}
\usepackage{array}
\usepackage{arydshln}
\newcommand{\wh}{\widehat}
\newcommand{\wt}{\widetilde}
\newcommand{\ifrac}{\delimpair{.}{/}{.}}
\usepackage{tkthm}

\title{Tight Risk Bounds for Gradient Descent on Separable Data}

\author{%
    Matan Schliserman\thanks{Blavatnik School of Computer Science, Tel Aviv University; \texttt{schliserman@mail.tau.ac.il}.}
    \and%
    Tomer Koren\thanks{Blavatnik School of Computer Science, Tel Aviv University, and Google Research; \texttt{tkoren@tauex.tau.ac.il}.}
}

\date{\today}

\begin{document}
\maketitle

\begin{abstract}
We study the generalization properties of unregularized
gradient methods applied to separable linear classification---a setting that has received considerable attention since the pioneering work of \citet{soudry2018implicit}.
We establish tight upper and lower (population) risk bounds for gradient descent in this setting, for any smooth loss function, expressed in terms of its tail decay rate.
Our bounds take the form 
$\Theta(r_{\ell,T}^2 / \gamma^2 T + r_{\ell,T}^2 / \gamma^2 n)$, 
where $T$ is the number of gradient steps, $n$ is size of the training set, $\gamma$ is the data margin, and $r_{\ell,T}$ is a complexity term that depends on the (tail decay rate) of the loss function (and on $T$).
Our upper bound matches the best known upper bounds due to \citet{shamir2021gradient,schliserman22a}, while extending their applicability to virtually any smooth loss function and relaxing technical assumptions they impose.
Our risk lower bounds are the first in this context and establish the tightness of our upper bounds for any given tail decay rate and in all parameter regimes.
The proof technique used to show these results is also markedly simpler compared to previous work, and is straightforward to extend to other gradient methods; we illustrate this by providing analogous results for Stochastic Gradient Descent.
\end{abstract}

\section{Introduction}
Recently, there has been a marked increase in interest regarding the generalization capabilities of unregularized gradient-based learning methods. 
One specific area of attention in this context has been the setting of linear classification with separable data, where a pioneering work by \citet{soudry2018implicit} showed that, when using plain gradient descent to minimize the empirical risk on a linearly separable training set with an exponentially-tailed classification loss (such as the logistic loss), the trained predictor will asymptotically converge in direction to the max-margin solution. 
As a result, standard margin-based generalization bounds for linear predictors suggest that, provided the number of gradient steps ($T$) is sufficiently large, the produced solution will not overfit, despite the lack of explicit regularization and the fact that its magnitude (i.e., Euclidean norm) increases indefinitely with $T$. This result has since been extended to incorporate other optimization algorithms and loss functions ~\citep{ji2018risk,ji2019refined,nacson2019convergence,nacson2019stochastic,ji2020regpath}.

Despite the high interest in this problem, the tight finite-time (population) risk performance of unregularized gradient methods, and even just of gradient descent, have not yet been fully understood.
The convergence to a high margin solution exhibited by \citet{soudry2018implicit} (for the logistic loss) occurs at a slow logarithmic rate, thus, the risk bounds for the trained predictors only become effective when $T$ is at least exponentially large in comparison to the size of the training set $n$ and the margin~$\gamma$. 
In a more recent work, \citet{shamir2021gradient} established risk bounds of the form $\wt{O}(\ifrac{1}{\gamma^2 T} + \ifrac{1}{\gamma^2 n})$ for several gradient methods in this setting (still with the logistic loss), that apply for smaller and more realistic values of $T$.
Later \citet{schliserman22a} gave a more general analysis 
that extends the bounds of \citet{shamir2021gradient} to a wide range of smooth loss functions satisfying a certain ``self-boundedness'' condition, and \citet{telgarsky22a} provided a test-loss analysis of stochastic mirror descent on ``quadratically-bounded'' losses.

However, all of these works either assume a specific loss (e.g., the logistic loss), impose various conditions on the loss function (beyond smoothness), or do not establish the tightness of their bounds in all regimes of parameters ($T, n$ and $\gamma$), and especially in the regime where the number of gradient steps $T$ is larger than the sample size $n$.

\subsection{Our contributions}
\label{contributions}

In this work, we close these gaps
by showing nearly matching upper and lower risk bounds for gradient descent, which hold essentially for \emph{any} smooth loss function, and for any number of steps $T$ and sample size $n$. 
Compared to recent prior work in this context \citep{shamir2021gradient,telgarsky22a,schliserman22a}, our results do not require any additional assumptions on the loss function besides its smoothness, and they strictly improve upon the existing bounds by their dependence on $T$. 
Further, to the best of our knowledge, risk lower bounds were not previously explored in this context, and the lower bounds we give establish, for the first time, the precise risk convergence rate of gradient descent for any smooth loss function and in all regimes of $T$ and $n$.

In some more detail, our results assume the following form. Let $\classtail$ be the class of nonnegative, convex and $\beta$-smooth loss functions $\ell(u)$ that decay to zero (as $u \to \infty$) faster than a reference ``tail function'' $\classfunc:[0,\infty)\to \R$. The function $\classfunc$ is merely used to quantify how rapidly the tails of loss functions in $\classtail$ decay to zero.
Then, the risk upper and lower bounds we prove for gradient descent are of the following form:%
\footnote{For simplicity, we specialize the bounds here to gradient descent with stepsize $\eta = \Theta(1/\beta)$, but the bounds hold more generally to any stepsize smaller than $O(1/\beta)$.}

\begin{align} \label{eq:bound-intro}
    \frac{\beta r_{\classfunc,T}^2}{\gamma^2 T}+\frac{\beta r_{\classfunc,T}^2}{\gamma^2 n}
    .
\end{align}
The bounds depend on the tail function $\classfunc$ through the term $r_{\classfunc,T}$ that, roughly, equals $\classfunc^{-1}(\epsilon)$ for $\epsilon$ chosen 
such that $\epsilon \approx \ifrac!{(\classfunc^{-1}(\epsilon))^2}{\gamma^2 T}$ %
(for the precise bounds refer to \cref{upper,lower}, and for concrete examples of implied bounds, see \cref{table:bounds}).
\begin{table}[t]
\small
\begin{center}
\begin{tabular}{lccccc} 
\toprule

 {\bfseries Tail decay rate}
& {\bfseries Risk bounds}

\\[0.5ex] 
\toprule
 $\exp(-x)$
&
$\displaystyle\Theta\left( \frac{\log^2(T)}{\margin^2T}+
     \frac{ \log^2(T)}{\margin^2n}\right)$
\\ 
 $x^{-\alpha}$
& $\displaystyle\Theta\brk2{\brk!{\frac{1}{\margin}}^{\frac{2\alpha}{2+\alpha}}
    \brk2{\frac{1}{T^{\frac{\alpha}{2+\alpha}}}+\frac{T^{\frac{2}{2+\alpha}}}{n}}}$  
\\
$\exp(-x^{\alpha})$ 
& 
$\displaystyle\Theta\left(\frac{ \log^{\frac{2}{\alpha}}(T)}{\margin^2 T}+\frac{\log^{\frac{2}{\alpha}}(T)}{\margin^2 n}\right)$
\\
\bottomrule
\end{tabular}
\end{center}
\caption{Examples of risk bounds established in this paper for gradient descent on $\gamma$-separable data, instantiated for several different loss tail decay rates.
Here, $T$ is the number of gradient steps and $n$ is the size of the training set.}
\label{table:bounds}
\end{table}

The form of the bound in \cref{eq:bound-intro} resembles the bounds given in recent work by \citet{schliserman22a}. However, they imposed an additional ``self-boundedness'' assumption that we do not require (for our upper bound), and they did not establish the tightness of their bounds, as we do in this paper by providing matching lower bounds.
On the flip side, their upper bounds apply in a broader stochastic convex optimization setup, whereas our bounds are specialized to generalized linear models for classification.
We also note that, for the specific case of an exponentially-tailed loss function, our rates match the upper bounds of \cite{shamir2021gradient} up to logarithmic factors.

In terms of upper bounds, our proof methodology is also distinctly (and perhaps somewhat surprisingly) simple compared to previous work.
We rely on two elementary properties that gradient methods admit when applied to a smooth and realizable objective: low training loss and low norm of the produced solutions. 
Both properties are obtained from fairly standard arguments and convergence bounds for smooth gradient descent, when paired with conditions implied by the decay rate of the loss function.
Finally, to bound the gap between the population risk and the empirical risk, we employ a classical result of \citet{srebrosmooth} that bounds the generalization gap of linear models in the so-called ``low-noise'' (i.e., nearly realizable) smooth regime using local Rademacher complexities.
Somewhat surprisingly, this simple combination of tools already give sharp risk results that improve upon the state-of-the-art~\citep{shamir2021gradient,schliserman22a,telgarsky22a} both in terms of tightness of the bounds and the light set of assumptions they rely on. 

We remark here that the proof scheme summarized above can be generalized to essentially any gradient method for which one can prove simultaneously bounds on the optimization error and the norm of possible solutions produced by the algorithm. 
In the sequel, we focus for concreteness on standard gradient descent, but in \cref{app_SGD} we also give bounds for Stochastic Gradient Descent (SGD), which is shown to admit both of these properties with high probability over the input sample.

For the lower bounds, our constructions are inspired by \citet{shamir2021gradient} who established a lower bound of $\Omega(\ifrac{1}{\gamma^2T})$ for the empirical risk of gradient descent with the logistic loss, which is tight up to logarithmic factors. 
When adapting their technique to other loss functions, e.g., with a polynomially decaying tail, these logarithmic factors become polynomial factors (in $T$) and the bound becomes no longer tight, even in the regime $T \ll n$ where the $1/T$ term in the bound is dominant. 
In contrast, using a careful adaptation of the lower bound construction, we establish nearly tight bounds for virtually any tail decay rate, and in all regimes of $T$ and $n$.

An interesting conclusion from our bounds pertains to the significance of early stopping.
We see that gradient descent, even when applied on a smooth loss functions that decay rapidly to zero, might overfit (i.e., reach a trivial $\Theta(1)$ risk) with respect the the surrogate loss function as the number of steps $T$ grows. 
The time by which gradient descent starts overfitting depends on the tail decay rate of the loss function: the slower the decay rate, the shorter the time it takes to overfit.
It is interesting to note that a similar phenomenon may not occur for the zero-one loss of the trained predictor. For example, \citet{soudry2018implicit} show that with the logistic loss, gradient descent does not overfit in terms of the zero-one as $T$ approached infinity---whereas, at the same time, it \emph{does} overfit as $T \to \infty$ in terms of the logistic loss itself, as seen from our lower bounds.

Another interesting aspect is the effect of a Lipschitz condition of the loss function on the achievable risk bounds.
Curiously, the upper bounds we are able to establish do not improve if we impose an additional Lipschitz assumption. An interesting direction for future work is to investigate whether this is a deficiency of our proof technique, or whether Lipschitzness of the loss function could be used to further improve the risk bounds of gradient methods.

\paragraph{Summary of contributions.}

To summarize, the main contribution of this paper are as follows: 
\begin{itemize}
    \item Our first main result (in \cref{sec:upper}) is a high probability risk upper bound for gradient descent in the setting of separable classification with a convex and smooth loss function. Our bound matches the best known upper bounds~\citep{shamir2021gradient,schliserman22a} and greatly extend them to allow for virtually any convex and smooth loss function, considerably relaxing various assumptions imposed by prior work.
    \item Our second main result (in \cref{sec:lower}) is a nearly matching lower bound for the risk of gradient descent, establishing the tightness of our analysis given the smoothness and tail-decay conditions. The tightness of our bounds holds across tail decay rates and different regimes of the parameters $T,n$ and $\gamma$.
    \item We also provide analogous results for Stochastic Gradient Descent with replacement (in \cref{app_SGD}), mainly to emphasize that that our analysis uses only two elementary properties of the optimization algorithm: low optimization error and low norm of the produced solution. The same analysis can be generalized to any gradient method that admits these two properties.
\end{itemize}

\subsection{Additional related work}

\paragraph{Unregularized gradient methods on separable data.}

The most relevant work to ours is of \citet{schliserman22a}, who used algorithmic stability and two simple conditions of self-boundedness and realizability which the loss functions hold, to get generalization bounds for gradient methods with constant step size in the general setting of stochastic convex and smooth optimization. Then, they derived risk bounds which hold in expectation for the setting of linear classification with separable data for every loss function which decays to $0$. The exact bound was depend in the rate of decaying to $0$ of the function. For the Lipschitz case their risk bound with respect to the loss function $\ell$ is 
$
    O( \ifrac!{\ell^{-1}(\epsilon)^2}{\gamma^2 T} + \ifrac!{\ell^{-1}(\epsilon)^2}{\gamma^2 n} )
$ 
for any choice of $\epsilon$ such that $\ifrac{\epsilon}{\ell^{-1}(\epsilon)^2}\leq \ifrac{1}{\gamma^2 T}$. For example, for the logistic loss, the bound translates to
$
    O( \ifrac!{\log^2(T)}{\gamma^2 T} + \ifrac!{\log^2(T)}{\gamma^2 n} )
    .
$ 

In another work, \citet{telgarsky22a}, showed a high probability risk bound for $T\leq n$ for Batch Mirror Decent with step size $\eta \simeq \ifrac{1}{\sqrt{T}}$ in linear models, using a reference vector, which when selected properly, can be translated to a risk bound of $O( \ifrac!{\ell^{-1}(\epsilon)^2}{\gamma^2 \sqrt{T}} )$ for gradient descent applied on the loss function $\ell$,
and to a $O( \ifrac!{\log^2 T }{\gamma^2 \sqrt{T}} )$ for the logistic loss. 

\paragraph{Fast rates for smooth and realizable optimization.}

The problem of smooth and realizable optimization, also known as the ``low-noise'' regime of stochastic optimization, is a very well researched problem. \citet{srebrosmooth} showed that stochastic gradient descent achieved risk bound of $O(\ifrac{1}{n})$ in this setting. 
For linear models, they also showed that ERM achieve similar fast rates by using local Rademacher complexities.
Later \cite{needle} showed that SGD converges linearly when the loss function is also strongly convex.
In more recent works, \cite{lei2020fine} used stability arguments to show that SGD with replacement with $T=n$ achieve risk of $O(\ifrac{1}{n})$.

\paragraph{Lower bounds.}

A lower bound related to ours appears in \citet{ji2019refined}.
In this work, the authors showed a lower bound of $\|w'_t-w^*\|\geq \ifrac{\log(n)}{\log(T)}$. In our work, however, we get lower bound directly for the loss itself and not for this objective.
More recently, \citet{shamir2021gradient} showed a lower bound of $\Omega(\ifrac{1}{\gamma^2T})$ for the empirical risk of GD when applied on the logistic loss which is tight up to log factors. When generalizing this technique for other objectives, e.g., functions that decay polynomially to zero, the log factors become polynomial factors and the bound becomes not tight, even in the regime $T \ll n$ where the $1/T$ term in the bounds is dominant. 
In contrast, we establish nearly tight bounds for virtually any tail decay rate, and in all regimes of $T$ and $n$. 

\section{Problem Setup}
\label{sec:setup}

We consider the following typical linear classification setting. 
Let $\D$ be distribution over pairs $(x,y)$, where $x \in \R^d$ is a $d$-dimensional feature vector and $y \in \R$ is  a real number that represents the corresponding label.
We focus on the setting of \emph{separable}, or \emph{realizable}, linear classification with margin.
Formally, we make the following assumption.
 
\begin{assumption}[realizability]
There exists a unit vector $w^* \in \R^d$ and $\gamma > 0$ such that $y (w^* \cdot x) \geq \margin$ almost surely with respect to the distribution $\D$.
\end{assumption}

Equivalently, we will identify each pair $(x,y)$ with the vector $z = yx$, and realizability implies that $w^* \cdot z \geq \margin$ with probability $1$.
We assume that data is scaled so that $\norm{z} \leq 1$ with probability $1$.

Given a nonnegative loss function $\loss:\R\to\R^+$, the objective is to determine a model $w \in \R^d$ that minimizes the (population) risk, defined as the expected value of the loss function over the distribution $\mathcal{D}$, namely
\begin{align} \label{eq:F_def}
    L(w) = \E_{z \sim \D}[\ell(w \cdot z)].
\end{align}
For finding such a model, we use a set of training examples $S=\{z_1,...,z_n\}$ which drawn i.i.d.\ from $\D$ and an empirical proxy, the \textit{empirical risk}, which is defined as
\begin{equation}
\label{F_hat_def}
    \wh{L}(w)
    =
    \frac1n \sum_{i=1}^n \ell(w\cdot z_i)
    .
\end{equation}

\subsection{Loss functions}

The loss functions $\ell$ considered in this paper are nonnegative, convex and $\beta$-smooth.%
\footnote{A function $\loss : \R \to \R$ is said to be $\beta$-smooth over $\R$ if $\loss(v) \leq \loss(u) + \loss'(u)\cdot (v-u)+\tfrac{1}{2}\beta(v-u)^2$ for all $u,v \in \R$.}
We also require that $\loss$ is strictly monotonically decreasing and $\lim_{u \to \infty} \loss(u) = 0$.
The vast majority of loss functions used in supervised learning for classification satisfy these conditions; these include, for example, the logistic loss ($\ell(u)=\log(1+e^{-u})$), the probit loss ($\loss(u)=-\log(\tfrac12 -\tfrac12\operatorname{erf}(u))$), and the squared hinge loss ($\ell(u)=(\max\{1-u,0\})^2$).   

A main goal of this paper is to quantify how do the achievable bounds on the risk depend on properties of the loss function $\ell$ used, and most crucially on the rate in which $\ell$ decays to zero as its argument approaches infinity.
To formalize this, we need a couple of definitions.

\begin{definition}[tail function]
    We say that $\classfunc:[0,\infty) \to \R$ is a \textit{tail function} if $\classfunc$
    \begin{enumerate}[nosep,label=(\roman*)]
        \item is a nonnegative, $1$-Lipschitz and $\beta$-smooth convex function;
        \item is strictly monotonically decreasing such that $\lim_{u \to \infty} \classfunc(u) = 0$;
        \item satisfies $\classfunc(0) \geq \frac{1}{2}$ and $|\classfunc'(0)|\geq\frac{1}{2}$. 
    \end{enumerate}
\end{definition}

Every tail function $\classfunc$ defines a class of loss functions characterized by the rate $\classfunc$ decays to zero. 

\begin{definition}[$\classfunc$-tailed class]
For a tail function $\classfunc$, the class $\classtail$ is the set of all nonnegative, convex, $\beta$-smooth and monotonically decreasing loss functions $\loss : \R \to \R^+$ such that $\loss(u)\leq \classfunc(u)$ for all $u\geq 0$.
\end{definition}

We detail several examples for tail functions in \cref{table:bounds}.

\subsection{Gradient Descent}

The algorithm that we focus in this paper is standard gradient descent (GD) with a fixed step size $\eta>0$ applied to the empirical risk $\widehat{L}$; this method is initialized at $w_1=0$ and at each step $t=1,\ldots,T$ performs an update
\begin{equation} \label{gd_update_rule}  
    w_{t+1} 
    = 
    w_t - \eta \nabla \widehat{L}(w_t)
    .
\end{equation}
The algorithm returns the final model, $w_T$.

We remark however that most of the results we present in the sequel can be straightforwardly adapted to other gradient methods; we include results for Stochastic Gradient Descent (SGD) in \cref{app_SGD}, and the same proof techniques can be used to analyze multi-epoch and/or mini-batched SGD, gradient flow, and more.

\section{Risk Upper Bounds}
\label{sec:upper}

We begin by giving a general upper bound for the risk of gradient descent, when the loss function~$\ell$ is taken from the class~$\classtail$.
Our main result in this section is the following.

\begin{theorem}
\label{upper}
Let $\classfunc$ be a tail function and let $\loss$ be any loss function from the class $\classtail$. 
Fix $T$,$n$ and $\delta>0$.
Then, with probability at least $1-\delta$ (over the random sample $S$ of size $n$), the output of GD applied on $\wh L$ with step size $\eta \leq\ifrac{1}{2\beta}$ initialized at $w_1=0$ has
\begin{align*}
    L(w_T)
    \leq 
    \frac{4K(\classfunc^{-1}(\epsilon))^2}{\gamma^2\eta T}
        +\frac{32K\beta (\classfunc^{-1}(\epsilon))^2 \brk!{\log ^{3}n + 4\log\frac{1}{\delta}}}{\gamma^2 n}
        +\frac{4K (\classfunc^{-1}(\epsilon))^2 \log\frac{1}{\delta}}{\gamma^2\eta T n}
\end{align*}
for any $\epsilon \leq \tfrac{1}{2}$ such that $\eta\gamma^2T\leq \ifrac{(\classfunc^{-1}(\epsilon))^2 }{\epsilon}$,
where $K < 10^5$ is a numeric constant.
\end{theorem}

Note that the expression $\ifrac{(\classfunc^{-1}(\epsilon))^2}{\epsilon}$ increases indefinitely as $\epsilon$ approaches 0; therefore, for any $T$, there exists an $\epsilon$ that satisfies the theorem's condition. For examples of how this bound is instantiated for different tail decay functions $\classfunc$, refer to \cref{table:bounds}.

In the remainder of this section we prove~\cref{upper}. The structure of the the proof will be as follows:
First, we bound the norm of the GD solution; by smoothness and realizability, we get that the norm will remain small compared to a reference point with small loss value. Second, we get a bound on the optimization error in this setting, the relies on the same reference point.
Finally, we use a fundamental result due to \citet{srebrosmooth} (reviewed in the subsection below) together with both bounds to derive the risk guarantee.
As discussed broadly in the introduction, this proof scheme can be generalized to other gradient methods which satisfy the properties of model with low norm and low optimization error.

\subsection{Preliminaries: Uniform Convergence Using Rademacher Complexity}

One property of linear models is that in this class of problems is that we have dimension-independent and algorithm-independent uniform convergence bounds, that enables to bound the difference between the empirical risk and the population risk of a specific model. A main technical tool for bounding this difference is the Rademacher Complexity~\citep{bartlett2002rademacher}.
The worst-case Rademacher complexity of an hypothesis class $H$ for any sample size $n$ is given by:
\begin{align*}
    R_n(H)=\sup_{z_1,...z_n}\E_{\sigma\sim\mathrm{Unif}\left(\{\pm1\}^n\right)}\left[~ \sup_{h\in H}\frac{1}{n}\left|\sum_{i=1}^n h(z_i)\sigma_i\right| ~\right]
    .
\end{align*}
We are interested in models that achieve low empirical risk on smooth objectives. A fundamental result of \citet{srebrosmooth} bounds the generalization gap under such conditions:
\begin{proposition}[\citealp{srebrosmooth}, Theorem 1]
\label{general_rademacher}
Let $H$ be a hypothesis class with respect to some non negative and $\beta$-smooth function, $\loss(t \cdot y)$, such that for every $w\in H,x,y$, $|\loss(wx\cdot y)|\leq b$. 
Then, for any $\delta>0$ we have, with probability at least $1-\delta$ over a sample of size $n$, uniformly for all $h\in H$,
\begin{align*}
    L(h)
    \leq
    \wh{L}(h) + K\left(\sqrt{\smash[b]{\wh{L}(h)}}\left(\sqrt{\beta}\log ^{1.5}(n) R_n(H)+\sqrt{\frac{b\log\frac{1}{\delta}}{n}}\right)+\beta\log^3(n) R_n^2(H)+\frac{b\log\frac{1}{\delta}}{n}\right).
\end{align*}
where $K < 10^5$ is a numeric constant.
\end{proposition}

\subsection{Properties of Gradient Descent on Smooth Objectives}

In this section we prove that GD satisfies the two desired properties- low norm and low optimization error. We begin with showing that the norm of $w_T$, the output of GD after $T$ iterations, is low, as stated in the following lemma,

\begin{lemma}
\label{norm_gd}
Let $\classfunc$ be a tail function and let $\loss\in \classtail$. Fix any $\epsilon>0$ and a point $w^*_\epsilon \in \R^d$ such that $\wh L(w^*_\epsilon) \leq \epsilon$ (exists due to realizability). 
Then, the output of $T$-steps GD, applied on $\widehat{L}$ with stepsize $\eta \leq \ifrac{1}{\beta}$ initialized at $w_1=0$ has,
 \begin{align*}
    \|w_{T}\|\leq 2\|w^*_\epsilon\|+2\sqrt{\eta \epsilon T}.
\end{align*}
\end{lemma}

\begin{proof}%
From $\beta$-smoothness, we know that $\|\nabla
\wh{L}(w)\|^2 \leq 2\beta\wh{L}(w)$ for any $w$ (see \cref{lem:2L_serbro} in \cref{app_upper}). 
Therefore, by using $\eta \leq \ifrac{1}{\beta}$, for every $\epsilon$,
\begin{align*}
    \|w_{t+1}-w^*_\epsilon\|^2
    &=
    \|w_t-\eta\nabla \widehat{L}(w_t)- w^*_\epsilon\|^2\\
    &=
 \|w_{t}-w^*_\epsilon\|^2-2\eta\langle w_{t}-w^*_\epsilon, \nabla \widehat{L}(w_t)\rangle+
\eta^2\|\nabla \widehat{L}(w_t)\|^2\\&\leq
 \|w_{t}-w^*_\epsilon\|^2+2\eta \widehat{L}(w^*_\epsilon)-2\eta \widehat{L}(w_t)+
2\beta\eta^2\widehat{L}(w_t)\\&\leq 
 \|w_{t}-w^*_\epsilon\|^2+2\eta \widehat{L}(w^*_\epsilon) 
 \\&\leq 
 \|w_{t}-w^*_\epsilon\|^2+2\eta \epsilon.
\end{align*}
By summing until time $T$,
\begin{align*}
    \|w_{T}-w^*_\epsilon\|^2&\leq \|w_1-w^*_\epsilon\|^2+2T\eta \epsilon 
    =\|w^*_\epsilon\|^2+2\eta \epsilon T.
\end{align*}
By taking a square root, using the fact that $\forall x,y\geq 0 \ \sqrt{x+y}\leq \sqrt{x}+\sqrt{y}$ and using triangle inequality,
\begin{align*}
    \|w_{T}\|
    =
    \|w_{T}-w^*_\epsilon\|+\|w^*_\epsilon\|\leq 2\|w^*_\epsilon\|+2\sqrt{\eta \epsilon T}
    .
\end{align*}
\end{proof}

Now, we bound the optimization error of GD on every function $\loss\in \classtail$, by using a variant of Lemma 13 from \cite{schliserman22a}. The proof is fairly standard and appears in \cref{app_upper}. 
\begin{lemma}
\label{opt_error}
Let $\classfunc$ be a tail function and let $\loss\in \classtail$. Fix any $\epsilon>0$ and a point $w^*_\epsilon \in \R^d$ such that $\wh L(w^*_\epsilon) \leq \epsilon$.  
Then, the output of $T$-steps GD, applied on $\wh L$ with stepsize $\eta \leq \ifrac{1}{\beta}$ initialized at $w_1=0$ has,
\begin{align*}
    \wh L(w_T)
    \leq \frac{\norm{w^*_\epsilon}^2}{\eta T} + 2\epsilon
    .
\end{align*}
\end{lemma}

\subsection{Proof of \cref{upper}}
\label{sec:GD}

We now turn to prove \cref{upper}. The proof is a simple consequence of the properties proved above and \cref{general_rademacher}.
We first claim that there exists a model $w^*_\epsilon$ with low norm such that $\wh{L}(w^*_\epsilon) \leq \epsilon$, which implies, through \cref{norm_gd,opt_error}, that $w_T$ of gradient descent has both low optimization error \emph{and} it remains bounded within a ball of small radius.
Then, we use \cref{general_rademacher} to translate the low optimization error to low risk.

\begin{proof}[of \cref{upper}]
First, we show that there exists a model $w^*_\epsilon$ with low norm such that $\wh{L}(w^*_\epsilon)\leq \epsilon$.
Let $\epsilon\leq \frac{1}{2}$ and $\loss\in \classtail$. 
By separability, there exists a unit vector $w^*$ such that $w^*\cdot  z_i\geq \margin$ for every $z_i$ in the training set $S$. Moreover, 
$\loss$ is monotonic decreasing. Then, for $w^*_\epsilon=\brk!{\ifrac{\classfunc^{-1}(\epsilon)}{\margin}} w^*$ and every $z_i\in S$,
\begin{align*}
    \loss(w^*_\epsilon \cdot z_i)= \loss\left(\frac{\classfunc^{-1}(\epsilon)}{\margin}w^*\ \cdot z_i\right) \leq \loss\left(\frac{\classfunc^{-1}(\epsilon)}{\margin}\cdot \margin\right)
    \leq \loss\brk!{\classfunc^{-1}(\epsilon)}.
\end{align*}
 Then, by the fact that $\classfunc^{-1}(\epsilon)\geq 0$, we have
 $
     \loss(w^*_\epsilon \cdot z_i)
     \leq \loss\brk{ \classfunc^{-1}(\epsilon) }
     \leq \classfunc\brk{ \classfunc^{-1}(\epsilon) }
     = \epsilon
$
for all $i$, hence
\begin{align*}
    \wh{L}(w^*_\epsilon)=\frac{1}{n}\sum_{i=1}^n \loss(w^*_\epsilon \cdot z_i)\leq \epsilon.
\end{align*}
Now, for $\epsilon$ such that $\eta\gamma^2T\leq \ifrac{(\classfunc^{-1}(\epsilon))^2 }{\epsilon}$, we get by \cref{norm_gd}, 
\begin{align*}
    \|w_{T}\|\leq 2\|w^*_\epsilon\|+2\sqrt{\eta \epsilon T}\leq     \frac{4\classfunc^{-1}(\epsilon)}{\margin}.
\end{align*}
For the same $\epsilon$, by \cref{opt_error},
\begin{align*}
     \widehat{L} (w_T)\leq \frac{\norm{w^*_\epsilon}^2}{\eta T} + 2\epsilon\leq 3\frac{\classfunc^{-1}(\epsilon)^2}{\gamma^2\eta T}
     .
\end{align*}
Denote $\normball=\{w: \|w\|\leq r_\epsilon\}$, where for brevity $r_\epsilon = \ifrac!{4\classfunc^{-1}(\epsilon)}{\gamma}$. 
We have, by \cref{bound_smooth}, $f(x) \leq 2f(y)+ \beta \norm{x-y}^2$ for all $x,y \in \R^d$ (see proof in \cref{app_upper}).
Then, together with the fact that  $\|z\|,\|z'\| \leq 1$ and choosing $\epsilon$ such that $\epsilon\leq \ifrac!{\classfunc^{-1}(\epsilon)^2}{\gamma^2\eta T}$, with probability $1$,
\begin{align*}
    b
    &=\max_{w\in \normball} \abs{\loss(w \cdot z)}
    \\
    &\leq 2\loss(w^*_\epsilon z) + 4\beta r_\epsilon^2 
    \\
    &\leq 2\epsilon + 4\beta r_\epsilon^2
    \\
    &\leq \frac{r_\epsilon^2}{8\eta T} + 4\beta r_\epsilon^2
    .
\end{align*}
Moreover, $\normball$ is hypothesis class of linear predictors with norm at most $r_\epsilon$.  We know that the norm of the examples is at most $1$, thus, it follows that the Rademacher complexity of $\normball$ is $R_n(\normball)=\ifrac!{r_\epsilon}{\sqrt{n}}$~\citep[e.g.,][Theorem 3]{kakade_ball}.

Now, by the choice of $\epsilon$, we have $\wh{L}(w_T)\leq \ifrac!{3r_\epsilon^2}{16\eta T}$.
Thus, \cref{general_rademacher} implies that with probability at least $1-\delta$, for every $w\in \normball$ and any $\epsilon$ such that $\epsilon\leq \ifrac!{\classfunc^{-1}(\epsilon)^2}{\gamma^2\eta T}$, 
\begin{align*}
     L(w_T)\leq\frac{3r_\epsilon^2}{16\eta T}+K\left(\sqrt{\frac{3r_\epsilon^2}{16\eta T}}\left(\sqrt{\beta}\log ^{1.5}n\frac{r_\epsilon}{\sqrt{n}}+\sqrt{\frac{b\log\frac{1}{\delta}}{n}}\right)+\beta\log^3n\frac{r_\epsilon^2}{n}+\frac{b\log\frac{1}{\delta}}{n}\right).
\end{align*}
Plugging in the bound on $b$, dividing by $r_{\epsilon}^2$ and using twice the fact that $xy \leq \tfrac{1}{2} x^2 + \tfrac{1}{2} y^2$ for all $x,y$,
\begin{align*}
    &\frac{L(w_T)}{r_\epsilon^2}
    \leq \frac{3}{16\eta T} + K\!\left(\!\sqrt{\frac{3}{16\eta T}}\!\left(\!\frac{\sqrt{\beta}\log ^{1.5}n}{\sqrt{n}}+\sqrt{\frac{(\frac{1}{8\eta T} + 4\beta)\log\frac{1}{\delta}}{n}}\right)\!
        +\frac{\beta\log^3n}{n}+\frac{(\frac{1}{8\eta T} + 4\beta)\log\frac{1}{\delta}}{n}\right)
    \\
    &\leq \frac{3}{16\eta T}+K\left(\sqrt{\frac{3}{16\eta T}}\left(\frac{\sqrt{\beta}\log ^{1.5}n}{\sqrt{n}}+\sqrt{\frac{\log\frac{1}{\delta}}{8\eta T n} +\frac{4\beta\log\frac{1}{\delta}}{n}}\right)+\frac{\beta\log^3n}{n}+\frac{\log\frac{1}{\delta}}{8\eta T n} +\frac{4\beta\log\frac{1}{\delta}}{n}\right)
    \\
    &\leq 
    \frac{3}{16\eta T}+K\left(\frac{3}{32\eta T}+\frac{1}{2}\left(\frac{\sqrt{\beta}\log ^{1.5}n}{\sqrt{n}}+\sqrt{\frac{\log\frac{1}{\delta}}{8\eta T n} +\frac{4\beta\log\frac{1}{\delta}}{n}}\right)^2+\frac{\beta\log^3n}{n}+\frac{\log\frac{1}{\delta}}{8\eta T n} +\frac{4\beta\log\frac{1}{\delta}}{n}\right)
    \\
    &\leq 
    \frac{3}{16\eta T}+K\left(\frac{3}{32\eta T}+\frac{\beta\log ^{3}n}{n}+\frac{\log\frac{1}{\delta}}{8\eta T n} +\frac{4\beta\log\frac{1}{\delta}}{n}+\frac{\beta\log^3n}{n}+\frac{\log\frac{1}{\delta}}{8\eta T n} +\frac{4\beta\log\frac{1}{\delta}}{n}\right)
    \\
    &\leq 
    \frac{7K}{32\eta T}+\frac{2K\beta
    \left(\log ^{3}n+4\log\frac{1}{\delta}\right)}{n}+\frac{K\log\frac{1}{\delta}}{4\eta T n}
\end{align*}
The theorem follows by rearranging the inequality.
\end{proof}

\section{Risk Lower Bounds}
\label{sec:lower}

In this section we present our second main result: a lower bound showing that the bound we proved in \cref{sec:upper} for gradient descent is essentially tight for loss functions in the class $\classtail$, for any given tail function $\classfunc$ and any $\beta>0$.
Formally, we prove the following theorem.

\begin{theorem} \label{lower}
There exists a constant $C$ such that the following holds.
For any tail function $\classfunc$, sample size $n \geq 35$ and any $T$, there exist a distribution $\D$ and a loss function $\loss\in\classtail$, such that for $T$-steps GD over a sample $S=\{z_i\}_{i=1}^n$ sampled i.i.d.~from $\D$, initialized at $w_1=0$ with stepsize $\eta \leq \ifrac{1}{2\beta}$, it holds that
\begin{align*}
        \E[ L(w_T) ]
        \geq C \frac{\beta (\classfunc^{-1}(128\epsilon))^2}{\gamma^2n} + C \frac{(\classfunc^{-1}(8\epsilon))^2}{\gamma^2\eta T} .
\end{align*}
for any $\epsilon \leq \frac{1}{256}$ such that $\eta\gamma^2T \geq \ifrac!{(\classfunc^{-1}(\epsilon))^2}{\epsilon}$.
\end{theorem}

We remark that the right-hand side of the bound is well defined, as we restrict $\epsilon$ to be sufficiently small so as to ensure that all arguments to $\classfunc^{-1}$ are at most $\tfrac{1}{2}$ (recall that $\classfunc$ admits all values in $[0,\tfrac12]$ due to our assumptions that $\phi(0) \geq \tfrac12$).
Further, the lower bound above matches the upper bound given in \cref{upper} up to constants, unless the tail function $\classfunc$ decays extremely slowly, and slower than any polynomial (at this point, however, the entire bound becomes almost vacuous).

To prove \cref{lower}, we consider two different regimes: the first is where $T \gg n$, when the first term in the right-hand side of the bound is dominant; and the $T \ll n$ regime where the second term is dominant.
We begin by focusing on the first regime, and prove the following.

\begin{lemma} \label{lower_bigt}
There exists a constant $C_1$ such that the following holds.
For any tail function $\classfunc$, sample size $n \geq 35$ and any $\gamma$ and $T$, there exist a distribution $\D$ with margin $\gamma$, a loss function $\loss\in\classtail$ such that for GD over a sample $S=\{z_i\}_{i=1}^n$  sampled i.i.d.~from $\D$, initialized at $w_1=0$ with stepsize $\eta \leq \ifrac{1}{2\beta}$, it holds that
\begin{align*}
    \E[L(w_T)] 
    \geq 
    C_1 \frac{\beta\classfunc^{-1}(128\epsilon)^2}{\gamma^2n}
    ,
\end{align*}
for any $\epsilon \leq \frac{1}{256}$ such that $\eta\gamma^2T \geq \ifrac!{(\classfunc^{-1}(\epsilon))^2}{\epsilon}$.
\end{lemma}

For the proof, we construct a hard learning problem for which the risk of GD can be lower bounded.
We define a loss function $\ell$ that, for $x>0$, decays to zero at the same rate as $\classfunc$ and, for $x \leq 0$, is a quadratic function.
The distribution $\D$ is constructed so that we have an example $z_1$ which appears frequently in the data set. In addition, there is another possible example $z_2$ which also appears frequently in the data set and is almost opposite to $z_1$, except having a small component which is orthogonal to $z_1$. Then, for achieving small optimization error, the GD iterate must have a significant component in the direction of $z_2$ which is orthogonal to $z_1$. 
Moreover, there is another (almost) opposite example $z_3$, that with constant probability, does \emph{not} appear in the training dataset. 
The bound is derived by the fact, if $z_3$ is sampled at test time (with probability roughly $1/n$), the error is quadratically large in the magnitude of the GD iterate. 

We remark that the norm of the GD iterate in the hard learning problem used in the proof is of order
$\Omega\left(\ifrac{\classfunc^{-1}(\epsilon)}{\gamma}\right)$. Then, we can conclude that the bound on the norm of the iterate that we give in the proof of \cref{upper} is tight up to constants.

\begin{proof}
Given $\gamma\leq \frac{1}{8}$, let us define the following distribution $\D$:
\[
    \D = 
    \begin{cases}
    z_1 := (1,0,0) 
    & \text{with prob.~$\frac{59}{64}(1-\frac{1}{n})$} ;
    \\
    z_2 := (-\frac{1}{2},3\gamma,0) 
    & \text{with prob.~$\frac{5}{64}(1-\frac{1}{n})$} ;
    \\
    z_3:=(0,-\frac{1}{8},4\gamma+\frac{1}{4}) 
    & \text{with prob.~$\frac{1}{n}$} ,
    \end{cases}
\] 
and loss function:
\[
    \loss(x) = 
    \begin{cases}
      \classfunc(x) 
      & x\geq 0 ;
      \\
      \classfunc(0) + \classfunc'(0) x + \frac{\beta}{2}x^2 
      & x<0.
    \end{cases}
\]
First, note that the distribution is separable:
for $w^*=(\gamma,\frac{1}{2},\frac{1}{4})$ it holds that 
$w^*z_i=\gamma$ for every~$i\in\{1,2,3\}$. 
Moreover, \cref{bigt_char} in \cref{sec:proofs-lower} ensures that indeed $\loss\in \classtail$. 

Next, let $S$ be a sample of $n$ i.i.d.~examples from $\D$ and let $z' \sim \D$ be a validation example independent from $S$. Denote by $\delta_2 \in [0,1]$ the fraction of appearances of $z_2$ in the sample $S$, and by $A_1,A_2$ the following events; 
\[
    A_1=\{z'=z_3 \wedge z_3\notin S\},
    \qquad
    A_2=\brk[c]!{\delta_2 \in \brk[s]!{\tfrac{1}{32},\tfrac{1}{8}}}
    .
\]
In \cref{totalprob} (found in \cref{sec:proofs-lower}), we show that
\begin{align} \label{eq:prA1A1}
    \Pr(A_1\cap A_2)\geq \frac{1}{120en}.
\end{align}
Furthermore, as in the proof of \cref{upper}, there exists a vector $w^*_\epsilon$ which holds $\|w^*_\epsilon\| \leq \ifrac!{\classfunc^{-1}(\epsilon)}{\gamma}$. 
Then by \cref{opt_error} and the choice of $\epsilon$,
\begin{align}
    \label{opt_lower_bigT}
    \wh{L}(w_T)\leq 2\epsilon + \frac{2\classfunc^{-1}(\epsilon)^2}{\gamma^2}\leq 4\epsilon.
\end{align}
For the remainder of the proof, we condition on the event $A_1 \cap A_2$.
First, we show that $w_t\cdot z_2\geq 0$.
Indeed, if it were not the case, then $\loss(w_T\cdot z_2)>\classfunc(0)$; together with \cref{opt_lower_bigT} we obtain
\begin{align*}
    \frac{1}{64}\geq 4\epsilon\geq \widehat{L}(w_T)>\delta_2\loss(w_T\cdot z_2)\geq \frac{1}{32} \classfunc(0).
\end{align*}
which is a contradiction to $\classfunc(0) \geq \tfrac{1}{2}$. 
Moreover, $w_T(1)\geq 0$. 
Again, we show this by contradiction. 
Conditioned on $A_2$, we have $\delta_1 >\frac{7}{8}$. Then,
    if $w_T(1)< 0$, $\loss (w_T\cdot z_1)> \classfunc(0)$, and
\begin{align*}
    \frac{1}{64}\geq 4\epsilon\geq\widehat{L}(w_T)\geq\delta_1\loss (w_T\cdot z_1)>\frac{7}{8}\classfunc(0)\geq \frac{7}{16}.
\end{align*}
which is a contradiction.
In addition, we notice that $z_3$ is the only possible example whose third entry is non zero. 
Given the event $A_1$, we know that $z_3$ is not in $S$. Equivalently, for every $z\in S$, $z(3)=0$.
    As a result,
    \begin{align*}
    w_t(3)=\eta \sum_{s=1}^{t-1}\nabla \widehat{L}(w_t)=\eta\sum_{s=1}^{t-1}\frac{1}{n}\sum_{z\in S}z(3)\loss'(w_sz)=0.
    \end{align*}
Then, we get that,
\begin{equation}
    \label{negative_w_t3}
    w_T\cdot z_3 = -\frac{1}{8}w_t(2).
\end{equation}
Then, using the fact that $w_T\cdot z_2\geq 0$, $\loss(w_T\cdot z_2)=\classfunc(w_T\cdot z_2)$, and conditioned on $A_2$, we have 
$$
    \loss(w_T\cdot z_2)
    = \classfunc(w_T\cdot z_2)
    \leq 32\widehat{L}(w_T))
    ,
$$
which implies
\begin{equation} \label{bigw_tz_2}
    w_T\cdot z_2\geq \classfunc^{-1}(32\widehat{L}(w_T)).
\end{equation}
  Therefore,  by combining \cref{bigw_tz_2} with the fact that $w_t(1)\geq 0$,
\begin{align*}
    3\gamma w_T(2)\geq -\frac{w_T(1)}{2}+3\gamma w_T(2)=w_T\cdot z_2 \geq \classfunc^{-1}(32\widehat{L}(w_T)).
\end{align*}
which implies,
$
    w_T(2)
    \geq \frac{1}{3\gamma} \classfunc^{-1}(32\widehat{L}(w_T))
    .
$
By \cref{negative_w_t3},
\begin{align*}
    w_T\cdot z_3= -\frac{1}{8} w_T(2) \leq -\frac{1}{24\gamma} \classfunc^{-1}(32\widehat{L}(w_T)).
\end{align*}
We therefore see that for every $\epsilon$ such that $\epsilon \geq \ifrac!{(\classfunc^{-1}(\epsilon))^2}{\gamma^2T\eta}$, 
\begin{align*}
    \loss(w_T\cdot z_3)
    &\geq \frac{\beta}{2}(w_T\cdot z_3)^2 
    \\
    &\geq \frac{\beta}{2}\left(\frac{1}{24\gamma} \classfunc^{-1}(32\widehat{L}(w_T))\right) ^2
    \\
    &\geq \frac{\beta}{1152\gamma^2} \brk!{ \classfunc^{-1}(32\widehat{L}(w_T)) }^2
    \\
    &\geq \frac{\beta}{1152\gamma^2} \brk!{ \classfunc^{-1}(128\epsilon) }^2,
\end{align*}
where in the final inequality we again used \cref{opt_lower_bigT}.
We conclude the proof using \cref{eq:prA1A1} and the law of total expectation,
\begin{align*}
    \E[L(w_T)]
    = \E[\loss(w_T\cdot z')]
    \geq \E[\loss(w_T\cdot z') \mid A_1\cap A_2] \Pr(A_1 \cap A_2)
    .
\end{align*}
(Expectations here are taken with respect to both the sample $S$ and the validation example $z'$.)
\end{proof}

Second, we show a lower bound for the second expression in the lower bound.  This expression is dominant in the early stages of optimization.

\begin{lemma} \label{lower_smallt}
There exists a constant $C_2$ such that the following holds.
For any tail function $\classfunc$, and for any $n,T$ and $\gamma$, there exist a distribution $\D$ with margin $\gamma$, a loss function $\loss\in\classtail$ such that for GD initialized at $w_1 = 0$ with stepsize $\eta \leq \ifrac{1}{2\beta}$ over an i.i.d.~sample $S=\{z_i\}_{i=1}^n$  from $\D$ and $w_1=0$ holds
\begin{align*}
    \E[L(w_T)] \geq C_2 \frac{(\classfunc^{-1}(8\epsilon))^2}{\gamma^2T\eta},
\end{align*}
for any $\epsilon \geq \tfrac{1}{16}$ such that $\eta\gamma^2T \leq \ifrac!{\classfunc^{-1}(\epsilon)^2}{\epsilon}$. 
\end{lemma}

The proof argument is similar to that of \cref{lower_bigt}.
A key difference is that the example that GD classifies incorrectly does appear in the dataset (though rarely).
We define a $1$-Lipschitz loss function $\ell$ that decays to zero at the same rate as $\classfunc$, and a distribution such that there is another possible example $z_1$ and an almost ``opposite'' example $z_2$, that with constant probability, appears limited times in the training samples $S$. 
The lower bound follows from the fact that although $z_2$ appears in the dataset, the gradients of the loss function are sufficiently large so as to make the trained predictor correct on $z_2$.

\begin{proof}
Given $\gamma\leq \frac{1}{8}$ and $\epsilon\leq \frac{1}{16}$, consider the following distribution;
\[\D=
\begin{cases}
      z_1:=(1,0) & \text{with prob.~$1-p$};\\
      z_2:=(-\frac{1}{2},3\gamma) & \text{with prob.~$p$},
    \end{cases}
\]
where $p = \frac{\classfunc^{-1}(8\epsilon)}{72\gamma^2T\eta}$.
Note that the distribution is separable, as for $w^*=\brk{\gamma,\frac{1}{2}}$ it holds that $w^*z_1=w^*z_2=\gamma$.
Further, consider the following loss function;
\[
    \loss(x)=
    \begin{cases}
        \classfunc(x) & \text{if $x\geq 0$;}
        \\
        \classfunc'(0)x+\classfunc(0) & \text{otherwise}.
    \end{cases}
\]
First, \cref{smallt_char} below ensures that $\loss \in \classtail$. 

Next, let $S$ be a sample of $n$ i.i.d.~examples from $\D$. Denote by $\delta_2 \in [0,1]$ the fraction of appearances of $z_2$ in the sample $S$, and by $A_1$ the event that $\delta_2 \leq 2p$.
By Markov's inequality, we have $\Pr(A_1) \geq \tfrac{1}{2}$.
Furthermore, as in the proof of \cref{upper}, there exists a vector $w^*_\epsilon$ which holds $\|w^*_\epsilon\|\leq \frac{\classfunc^{-1}(\epsilon)}{\gamma}$. 
Then by \cref{opt_error} and the choice of $\epsilon$,
\begin{align}
    \label{opt_lower_smallT}
    \wh{L}(w_T)\leq 2\epsilon + \frac{2\classfunc^{-1}(\epsilon)^2}{\gamma^2}\leq 4\epsilon.
\end{align}
Now, we assume that $A_1$ holds.
We know that 
\begin{align} \label{eq:delta2}
    \delta_2
    \leq 2p
    \leq
    \frac{\classfunc^{-1}(8\epsilon)}{36\gamma^2T\eta}
    \leq \epsilon
    \leq \frac{1}{2}
    ,
\end{align}
thus, conditioned on $A_1$ and by \cref{opt_lower_smallT},
\begin{align} \label{eq:4eps}
     4\epsilon
     \geq 
     \widehat{L}(w_T)> (1-\delta_2)\loss(w_T\cdot z_1) 
     \geq 
     \frac{1}{2}\loss(w_T(1))
     .
\end{align}
If $w_T(1)<0$, we get that
\begin{align*}
     4\epsilon> \frac{1}{2}\loss(0)=\frac{1}{2}\classfunc(0)\geq\frac{1}{4}
\end{align*}
which is a contradiction to our assumption that $\epsilon \leq \frac{1}{16}$. 
Then $w_T(1) \geq 0$ and from \cref{eq:4eps},
$
    8\epsilon 
    \geq
    \loss(w_T(1))
    =
    \classfunc(w_T(1)).
$
which implies that
\begin{align} \label{eq:wT1}
    w_T(1)\geq \classfunc^{-1}(8\epsilon)
    .
\end{align}
Now, by the fact that $\classfunc'(0)\leq 1$ it follows that $\loss$ is $1$-Lipschitz. 
Then, from the GD update rule,
\begin{align*}
    w_{t+1}(2)
    =
    w_t(2) - 3\eta \cdot \gamma \delta_2 \loss'(w_t\cdot z_2)
    \leq 
    w_t(2) + 3\gamma\delta_2 \eta
    ,
\end{align*}
from which it follows that
\begin{align} \label{eq:wT2}
    w_{T}(2) \leq 3\gamma\delta_2 \eta T.
\end{align}
From \cref{eq:delta2,eq:wT1,eq:wT2} we now obtain that
\begin{align*}
    w_T\cdot z_2 
    &\leq 9\gamma^2\delta_2 T \eta- \frac{1}{2}\classfunc^{-1}(8\epsilon) 
    \\
    &\leq 9\gamma^2T\eta\frac{\classfunc^{-1}(8\epsilon)}{36\gamma^2T\eta}- \frac{1}{2}\classfunc^{-1}(8\epsilon) 
    \\
    &= -\frac{1}{4}\classfunc^{-1}(8\epsilon) 
    .
\end{align*}
By the fact that $\forall x<0 : \loss(x)\geq -\frac{1}{2}x$, this implies that in the event $A_1$ it holds that:
\begin{align} \label{eq:wTz2}
    \loss(w_T\cdot z_2)
    \geq 
    -\frac{1}{2}w_T\cdot z_2 
    \geq 
    \frac{1}{8}\classfunc^{-1}(8\epsilon) 
    .
\end{align}
Finally, for a new validation example $z' \sim \D$ (independent from the sample $S$), 
\begin{align} \label{eq:Prz2A1}
    \Pr(\{z'=z_2\}\cap A_1)
    = \Pr(z'=z_2 \mid A_1) \Pr(A_1)
    \geq \frac{1}{2}P(z'=z_2)
    = \frac{1}{2} p
    \geq \frac{\classfunc^{-1}(8\epsilon)}{144\gamma^2T\eta}
    .
\end{align}
To conclude, from \cref{eq:wTz2,eq:Prz2A1} we have
\begin{align*}
    \E[\loss(w_Tz')]
    &\geq \E[ \loss(w_Tz') \mid \{z'=z_2\}\cap A_1] \Pr(\{z'=z_2\}\cap A_1) \\
    &\geq \frac{\classfunc^{-1}(8\epsilon)}{144\gamma^2T\eta} \cdot \frac{1}{8}\classfunc^{-1}\left(8\epsilon\right) \\&= \frac{\classfunc^{-1}(8\epsilon)^2}{1152\gamma^2T\eta}
    .%
    \end{align*}
\end{proof}

\cref{lower} now follows directly from \cref{lower_bigt,lower_smallt}:

\begin{proof}[of \cref{lower}]
    Let $C=\frac{1}{2}\min\brk[c]{ C_1,C_2 }$, where $C_1$ and $C_2$ are the constants from \cref{lower_bigt,lower_smallt}, respectively.
    If $\ifrac!{(\classfunc^{-1}(8\epsilon))^2}{\gamma^2T\eta}\geq \ifrac!{\beta(\classfunc^{-1}(128\epsilon))^2}{\gamma^2n}$, the theorem follows from \cref{lower_smallt}; otherwise, it follows from \cref{lower_bigt}.
\end{proof}

\subsection*{Acknowledgments}
This work has received support from the Israeli Science Foundation (ISF) grant no. 2549/19, the Len Blavatnik and the Blavatnik Family Foundation and the Yandex Initiative in Machine Learning.
\bibliography{sample.bib}

\newpage
\appendix
\section{Upper bound for Stochastic Gradient Descent With Replacement}
\label{app_SGD}

We also show generalization bound for Stochastic Gradient Descent (SGD) which is a randomized algorithm which achieves low optimization error in high probability.

Given a dataset $S=\{z_1,...z_n\}$, we define $\loss_i(w)=\loss(w\cdot z_i)$.
Then, SGD with replacement is
initialized at a point $w_1=0$ and at each step $t=1,\ldots,T$, samples
randomly an index $i_t\in[n]$ and performs an update
\begin{equation}
\label{sgd_update_rule}
    w_{t+1} = w_t - \eta \nabla \loss_{i_{t}}(w_t),
\end{equation}
where $\eta>0$ is the step size of the algorithm. We consider a standard variant of SGD that returns the average iterate, namely $\overline{w}_T=\frac{1}{T}\sum_{t=1}^{T}w_t$.
We start with showing a deterministic bound on the norm of the iterate of the algorithm. Then, we show high probability bound for the empirical risk $\wh{L}$.  Then, the proof  of the risk bound is identical to the proof of the risk bound for GD (see \cref{upper}). As a result, we will not show the full proof, but only the bounds on the norm (see \cref{norm_sgd}) and the empirical risk (see \cref{opt_error_sgd}).

\begin{lemma}
\label{norm_sgd}
Let $T$. Let $\classfunc$ be a tail function and let $\loss\in \classtail$. Assume that for every $\epsilon>0$ there exists a point $w^*_\epsilon$ such that for every $i$, $\loss_i(w^*_\epsilon) \leq \epsilon$. Then, the output of SGD with replacement, applied on $\widehat{L}$ with step size
$\eta \leq \frac{1}{\beta}$ initialized at $w_1=0$ has,
 \begin{align*}
    \|\bar{w}_{T}\|\leq 2\|w^*_\epsilon\|+2\sqrt{\eta \epsilon T}.
\end{align*}
\end{lemma}
\begin{proof}
 By \cref{lem:2L_serbro} (see \cref{app_upper}), we know that for every $w$, $\|\nabla
\loss_{i_{t}}(w)\|^2 \leq 2\beta\loss_{i_{t}}(w)$. Therefore, by using $\eta \leq \frac{1}{\beta}$, for every $\epsilon$,
\begin{align*}
    \|w_{t+1}-w^*_\epsilon\|^2
    &=
    \|w_t-\eta\nabla \loss_{i_{t}}(w_t)- w^*_\epsilon\|^2\\
    &=
 \|w_{t}-w^*_\epsilon\|^2-2\eta\langle w_{t}-w^*_\epsilon, \nabla \loss_{i_{t}}\rangle+
\eta^2\|\nabla \loss_{i_{t}}(w_t)\|^2\\&\leq
 \|w_{t}-w^*_\epsilon\|^2+2\eta \loss_{i_{t}}(w^*_\epsilon)-2\eta \loss_{i_{t}}(w^*_\epsilon)+
2\beta\eta^2\loss_{i_{t}}(w_t)\\&\leq 
 \|w_{t}-w^*_\epsilon\|^2+2\eta \loss_{i_{t}}(w^*_\epsilon) 
 \\&\leq 
 \|w_{t}-w^*_\epsilon\|^2+2\eta \epsilon.
\end{align*}
By summing until time $T$,
\begin{align*}
    \|w_{T}-w^*_\epsilon\|^2&\leq \|w_1-w^*_\epsilon\|^2+2T\eta \epsilon 
    \\&=\|w^*_\epsilon\|^2+2\eta \epsilon T.
\end{align*}
By taking a square root, using the fact that $\forall x,y\geq 0 \ \sqrt{x+y}\leq \sqrt{x}+\sqrt{y}$,
\begin{align}
\label{distanceopt_sgd}
    \|w_{T}-w^*_\epsilon\|\leq\|w^*_\epsilon\|+2\sqrt{\eta \epsilon T}.
\end{align}
and by using triangle inequality,
\begin{align*}
    \|w_{T}\|
    =
    \|w_{T}-w^*_\epsilon\|+\|w^*_\epsilon\|\leq 2\|w^*_\epsilon\|+2\sqrt{\eta \epsilon T}.
\end{align*}
and finally by another use of traingle inequality,
\begin{align*}
    \|\bar{w}_{T}\|
    \leq \frac{1}{T}\sum_{t=1}^T\|w_{T}\|\leq 2\|w^*_\epsilon\|+2\sqrt{\eta \epsilon T}.
\end{align*}
\end{proof}
Now, we prove the next lemma and bound the regret of $SGD$ with replacement.
\begin{lemma}
\label{regret_sgd}
    Let $T$. Let $\classfunc$ be a tail function and let $\loss\in \classtail$. Then, the iterate of SGD with replacement, applied on $\widehat{L}$ with step size
$\eta \leq \frac{1}{2\beta}$ initialized at $w_1=0$ has, for every $w\in R^d$,
\begin{align*}
    \frac{1}{T}\sum_{t=1}^T\loss(w_t\cdot z_{i_{t}})-\frac{1}{T}\sum_{t=1}^T\loss(w\cdot z_{i_{t}})\leq \frac{\|w\|^2}{2\eta T}.
\end{align*}
\end{lemma}
\begin{proof}
The proof is almost the same as of \cite{schliserman22a}.
For every $w$, iteration $j$ and possible $i_j$, by \cref{lem:2L_serbro} and convexity, 
\begin{align*}
    \|w_{j+1}-w\|^2&\leq \|w_{j}-w\|^2 -2\eta \langle \nabla\loss_{i_{j}}(w_j)(w_j-w)\rangle +\eta^2\|\nabla\loss_{i_{j}}(w_j)\|^2
    \\&\leq 
    \|w_{j}-w\|^2 +2\eta  \loss(w\cdot z_{i_{j}})-2\eta \loss(w_j\cdot z_{i_{j}}) +2\eta^2L \loss(w_j\cdot z_{i_{j}})
    \\&\leq 
    \|w_{j}-w\|^2 +2\eta  \loss(w\cdot z_{i_{j}})- 2\eta  \loss (w_j\cdot z_{i_{j}})
    .
\end{align*}

Taking average on $j=1...{T}$, we get
\begin{align*}
    \frac{1}{T}\sum_{t=1}^T\loss(w_t\cdot z_{i_{t}})-\frac{1}{T}\sum_{t=1}^T\loss(w\cdot z_{i_{t}})\leq \frac{\|w\|^2}{2\eta T}.
\end{align*}
\end{proof}
We use the following concentration bound, taken from  \cite{agarwal2014taming}, Lemma 9.
\begin{lemma}\label{bagel}
Suppose $Z_1,...Z_T$ is a sequence such that for every $t\leq T$, $\E\left(Z_t|Z_1...Z_{t-1}\right)=0$. and let $Z=\sum_{t=1}^T Z_t$. Assume that $|Z_t|\leq b$ for all $t$, and define $V=\sum_{t=1}^T \E\left[Z_t^2 \mid \mathcal{F}_{t-1}\right]$. Then, for any $\delta>0$ and $\lambda\in[0,\ifrac{1}{b}]$ , with probability of $1-\delta$,
\begin{align*}
    Z\leq \lambda V  + \frac{1}{\lambda}\log\frac{1}{\delta}
    .
\end{align*}
\end{lemma}

Then, we can get high probability guarantee for SGD with replacement. 
\begin{lemma}
\label{opt_error_sgd}
Let $T$ and $\delta>0$. Let $\classfunc$ be a tail function and let $\loss\in \classtail$. Assume that for every $\epsilon>0$ there exists a point $w^*_\epsilon$ such such that for every $i$, $\loss_i(w^*_\epsilon) \leq \epsilon$. Then, the output of SGD with replacement, applied on $\widehat{L}$ with step size
$\eta \leq \frac{1}{\beta}$ initialized at $w_1=0$ has, with probability of $1-\delta$,
\begin{align*}
     \widehat{L}\left(\bar{w}_T\right)&\leq \frac{\|w^*_\epsilon\|^2}{\eta T}+
    3\epsilon + \frac{8\left(3\epsilon + 16\beta \|w^*_\epsilon\|^2+16\eta \epsilon T\right)}{T}\log\left(\frac{1}{\delta}\right).
\end{align*}
\end{lemma}
\begin{proof}
Let $w$. 
We define 
\[Z_t=\loss(w^*_\epsilon \cdot z_{i_t}) -\widehat{L}(w^*_\epsilon)-\loss(w_t\cdot z_{i_t})+\widehat{L}(w_t).\]
By the fact that $w_t,i_t$ is independent,  $\E\left(Z_t|Z_1...Z_{t-1}\right)=0$.
First,
\begin{align*}
    b&=\max_{t}|Z_t|\\&\leq \epsilon+\max_{t,i} \loss(w_t\cdot z_i)
    \tag{$\loss(w^*_\epsilon \cdot z_{i_t})\leq \epsilon$, definition of $\wh{L}$, nonnegativity}
    \\&\leq \epsilon +\max_{i,\|w\|\leq 2\|w^*_\epsilon\|+2\sqrt{\eta \epsilon T}}\loss(w\cdot z_i)
    \tag{\cref{norm_sgd}}
    \\&\leq \epsilon+ 2\max_i\loss(w^*_\epsilon\cdot z_i) + 16\beta \|w^*_\epsilon\|^2+16\eta \epsilon T \tag{\cref{bound_smooth}}
    \\&\leq 3\epsilon + 16\beta \|w^*_\epsilon\|^2+16\eta \epsilon T.
\end{align*}
We denote $b=3\epsilon + 16\beta \|w^*_\epsilon\|^2+16\eta \epsilon T$. Then, 
\begin{align*}
\max_{t}|Z_t|\leq 3\epsilon + 16\beta \|w^*_\epsilon\|^2+16\eta \epsilon T=b.
\end{align*}
Moreover, we denote $\E_t[\cdot]$ the expectation conditioned on the randomness of the algorithm before step $t$.
Then, 
\begin{align*}
    V&=\sum_{t=1}^T \E_t\left[\left(\loss(w_t\cdot z_{i_t}) -\widehat{L}(w_t)-\loss(w\cdot z_{i_t})+\widehat{L}(w)\right)^2\right]
    \\&\leq
    b\sum_{i=1}^T\E_t\left[|\loss(w_t\cdot z_{i_t}) -\widehat{L}(w_t)-\loss(w\cdot z_{i_{t}})+\widehat{L}(w)|\right]
    \\&\leq
    b\sum_{i=1}^T\E_t\left[\loss(w_t\cdot z_{i_t})|\right]
    + b\sum_{i=1}^T\E_t\left[\loss(w\cdot z_{i_t})|\right] +b\sum_{i=1}^T\E_t\left[\widehat{L}(w_t)|\right]
    +b\sum_{i=1}^T\E_t\left[\widehat{L}(w)|\right]
     \\&\leq 
    2b\sum_{i=1}^T\widehat{L}(w_t)
    +2bT\sum_{i=1}^T\widehat{L}(w)
\end{align*}
By \cref{bagel}, with probability $1-\delta$ by choosing $\lambda =\frac{1}{4b}$,
\begin{align*}
    \sum_{t=1}^T Z_t \leq \frac{1}{2}\sum_{i=1}^T\widehat{L}(w_t)+
    \frac{1}{2}T\widehat{L}(w) + \frac{4b}{T}\log\left(\frac{1}{\delta}\right).
\end{align*}
Then, with probability $1-\delta$,
\begin{align*}
    \frac{1}{T}\sum_{i=1}^T\widehat{L}(w_t)-\widehat{L}(w)&=\frac{1}{T}\sum_{t=1}^T\loss(w_t\cdot z_{i_t}) -\loss(w\cdot z_{i_t})+\frac{1}{T}\sum_{t=1}^T Z_t
    \\&\leq \frac{1}{T}\sum_{i=1}^T\loss(w_t\cdot z_{i_t}) -\frac{1}{T}\sum_{i=1}^T\loss(w\cdot z_{i_{t}})+ \frac{1}{2T}\sum_{t=1}^T\widehat{L}(w_t)+
    \frac{1}{2}\widehat{L}(w) + 4b\log\left(\frac{1}{\delta}\right).
\end{align*}
By organizing, 
\begin{align*}
    \frac{1}{2T}\sum_{i=1}^T\widehat{L}(w_t)
    \leq \frac{1}{T}\sum_{i=1}^T\loss(w_t\cdot z_{i_t}) -\frac{1}{T}\sum_{i=1}^T\loss(w\cdot z_{i_{t}})+
    \frac{3}{2}\widehat{L}(w) + \frac{4b}{T}\log\left(\frac{1}{\delta}\right).
\end{align*}
Moreover, with probability $1-\delta$, by \cref{regret_sgd} and Jensen inequality
\begin{align*}
\widehat{L}\left(\frac{1}{T}\sum_{i=1}^Tw_t\right)\leq \frac{1}{T}\sum_{i=1}^T\widehat{L}(w_t) &\leq\frac{2}{T}\sum_{i=1}^T\loss(w_t\cdot z_{i_t}) -\frac{2}{T}\sum_{i=1}^T\loss(w\cdot z_{i_{t}})+
    3\widehat{L}(w) + \frac{8b}{T}\log\left(\frac{1}{\delta}\right)
    \\&\leq \frac{\|w\|^2}{\eta T}+
    3\widehat{L}(w) + \frac{8b}{T}\log\left(\frac{1}{\delta}\right).
\end{align*}
Finally, for $w=w^*_\epsilon$,
\begin{align*}
    \widehat{L}\left(\bar{w}_T\right)&\leq \frac{\|w^*_\epsilon\|^2}{\eta T}+
    3\epsilon + \frac{8\left(3\epsilon + 16\beta \|w^*_\epsilon\|^2+16\eta \epsilon T\right)}{T}\log\left(\frac{1}{\delta}\right).
\end{align*}
\end{proof}

\section{Proofs of \cref{sec:upper}}
\label{app_upper}

We rely on the following standard lemma about smooth functions (proof can be found in, e.g., \citealp{nesterov2003introductory}, or in \citealp{srebrosmooth}).

\begin{lemma} \label{lem:2L_serbro}  
For a non-negative and $\beta$-smooth $f:\R^d \to \R$, it holds that $\|\nabla
f(w)\|^2 \leq 2\beta f(w)$ for all $w\in \R^d$.
\end{lemma}

\begin{proof}[of \cref{opt_error}]
The proof follows the argument of \cite{schliserman22a}.
First, by $\beta$-smoothness, for every $t$ and $\eta \leq \ifrac{1}{\beta}$,
\begin{align*}
    \widehat{L}(w_{t+1})&\leq \widehat{L}(w_t) + \nabla \widehat{L}(w_t) \cdot (w_{t+1}-w_t) +\frac{\beta}{2} \|w_{t+1}-w_t\|^2
    \\&=
    \widehat{L}(w_t) - \eta  \|\nabla \widehat{L}(w_t)\|^2  + \frac{\eta^2 \beta}{2} \|\nabla \widehat{L}(w_t)\|^2
    \\&\leq
    \widehat{L}(w_t) - \frac{\eta}{2}\|\nabla \widehat{L}(w_t)\|^2
    \\&\leq
    \widehat{L}(w_t).
\end{align*}
Hence,
\begin{equation}
\label{GD_mono}
    \widehat{L}(w_T)\leq \frac{1}{T} \sum_{t=1}^{T}\widehat{L}(w_t).
\end{equation}
Moreover, from standard regret bounds for gradient updates, for any $w \in \R^d$,
\begin{align*}
    \frac{1}{T} \sum_{t=1}^{T} \brk{ \widehat{L}(w_t) - \widehat{L}(w) }
    &\leq
    \frac{\norm{w_1-w}^2}{2\eta T} + \frac{\eta}{2T} \sum_{t=1}^{T} \norm{\nabla \widehat{L}(w_t)}^2
    .
\end{align*}
By \cref{lem:2L_serbro},
\begin{align*}
    \frac{1}{T} \sum_{t=1}^{T} \brk{ \widehat{L}(w_t) - \widehat{L}(w) }
    &\leq
    \frac{\norm{w}^2}{2\eta T} + \frac{\eta \beta}{T} \sum_{t=1}^{T} \widehat{L}(w_t)
    .
\end{align*}
Using $\eta \leq \ifrac{1}{2\beta}$ gives
\begin{align*}
    \frac{1}{T} \sum_{t=1}^{T} \widehat{L}(w_t)
    &\leq
    \frac{\norm{w}^2}{\eta T} + 2\widehat{L}(w)
    .
\end{align*}
For $w=w^*_\epsilon$ we get by \cref{GD_mono},
 \begin{align*}
    \widehat{L}(w_T)\leq \frac{1}{T}\sum_{t=1}^{T} \widehat{L}(w_t)
    \leq 
    \frac{\norm{w^*_\epsilon}^2}{\eta T} + 2\widehat{L}(w^*_\epsilon)
    \leq
     \frac{\norm{w^*_\epsilon}^2}{\eta T} + 2\epsilon
    .
 \end{align*}
\end{proof}
\begin{lemma}
    \label{bound_smooth}
    Let $f:\R ^d\to \R$ be a $\beta$-smooth and nonnegative function. Then $f(x) \leq 2f(y)+ \beta \norm{x-y}^2$ for all $x,y \in \R^d$.
\end{lemma}
\begin{proof}
    For any $x,y \in \R^d$:
    \begin{align*}
    f(x)
    &\leq f(y)+ \nabla f(y) \cdot (x-y)+\frac{\beta}{2} \|x-y\|^2
    \tag{$\beta$-smoothness}\\
    &\leq f(y)+ \frac{1}{2\beta}\|\nabla f(y)\|^2+\frac{\beta}{2}\|x-y\|^2+\frac{\beta}{2} \|x-y\|^2
    \tag{$\forall c>0 ~:~ xy\leq \frac{1}{2c}x^2+\frac{c}{2} y^2$}
    \\
    &\leq 2f(y)+\beta\|x-y\|^2
    \tag{\cref{lem:2L_serbro}}
    .
\end{align*}
\end{proof}
\section{Proof of \cref{sec:lower}}
\label{sec:proofs-lower}

\subsection{Proof of \cref{lower_bigt}}
\label{app_lower_bigt}

\begin{lemma}
\label{bigt_char}
Let $\classfunc$ be a tail function. Let $\loss(x)$ be the following function,
\[\loss(x)=
\begin{cases}
      \classfunc(x) & \text{if $x\geq 0$;} \\
      \classfunc(0)+x\classfunc'(0)+\frac{\beta}{2}x^2 & \text{if $x<0$.}
    \end{cases}
\]
Then, $\loss\in \classtail$.
\end{lemma}
\begin{proof}
First, it is easy to verify that $\loss$ is continuously differentiable.
Second, $\loss$ is non negative: for $x\geq0$ by the non negativity of $\classfunc$ and for $x<0$ by the fact that $\classfunc'(0)\leq 0$.
Moreover, $\loss$ is convex. We need to prove that every $x<y$, $\loss'(x)\leq \loss'(y)$ For $x,y<0$, we get it by the convexity of $\classfunc$.  For $x,y>0$, we get it by the fact $\loss$ there is a sum of convex function and linear function. For $x<0<y$, by the convexity of $\classfunc$,
\begin{align*}
    \loss'(x)=\classfunc'(0)+\beta x\leq \classfunc'(0)\leq \classfunc'(y).
\end{align*}
In addition, $\loss$ is $\beta$-smooth. We need to prove that every $x<y$, $\loss'(y)-\loss'(x)\leq \beta(y-x)$ For $x,y\geq0$, we get it by the smoothness of $\classfunc$.  For $x,y\leq0$, we get it by the fact that $\loss$ is a sum of $\beta$-smooth function and a linear function. For $x\leq0\leq y$, by the smoothness of $\classfunc$,
\begin{align*}
    \loss'(y)-\loss'(x)=\classfunc'(y)-\classfunc'(0)-\beta x\leq \beta(y-x).
\end{align*}
Finally, $\loss$ is strictly monotonically decreasing.
We need to prove that every $x<y$, $\loss(y)>\loss(x)$. For $x,y>0$, we get it by the monotonicity of $\classfunc$.  For $x<y<0$, 
\begin{align*}
    \loss(y)=\classfunc(0)+\classfunc'(0)y+\frac{\beta}{2}y^2\leq \classfunc(0)+\classfunc'(0)x+\frac{\beta}{2}x^2=\loss(x).
\end{align*}
For $x<0<y$, 
\begin{align*}
    \loss(y)=\classfunc(y)\leq \classfunc(0) \leq \classfunc(0)+\classfunc'(0)x+\frac{\beta}{2}x^2=\loss(x)
    .
\end{align*}
\end{proof}

\begin{lemma}
\label{totalprob}
Let $S \sim \D^n$ be a sample of size $n$, and let $z' \sim \D$ be a validation example. Moreover, Assume $n\geq 35$ and let $\prob_2$  be the fraction of $z_2$ in $S$.
We define the following event, 
\[A=\{z_3\notin S\} \cap\{z'=z_3\} \cap\{\delta_2\in [ \tfrac{1}{32}, \tfrac{1}{8} ]\}.\]
Then, \begin{align*}
    \Pr(A)\geq \frac{1}{120en}
\end{align*}
\end{lemma}
\begin{proof}
The proof follows directly by \cref{A_1_nonlip} and \cref{A_2_nonlip}.
We define the following events,
\[A_1=\{z_3\notin S\} \cap\{z'=z_3\},A_2=\{\delta_2\in [ \tfrac{1}{32}, \tfrac{1}{8} ]\}.\]
By \cref{A_1_nonlip}.
\begin{align*}
    \Pr(A_1) \geq \frac{1}{2en}.
\end{align*}
By $\cref{A_2_nonlip}$,
\begin{align*}
     \Pr(A_2|A_1)\geq \frac{1}{60}.
\end{align*}
Then, combining both results, 
\begin{align*}
    \Pr(A)\geq\Pr(A_1) \Pr(A_2 \mid A_1)\geq\frac{1}{120en}
    .
\end{align*}
\end{proof}

\begin{lemma}
\label{A_1_nonlip}
Let $S \sim \D^n$ be a sample of size $n$, and let $z' \sim \D$ be a validation example.
 Then,
\begin{align*}
    \Pr(A_1) 
    = \Pr(z_3\notin S \wedge z'=z_3) 
    \geq \frac{1}{2en}.
\end{align*}
\end{lemma}
\begin{proof}
First, we know that,
\begin{align*}
    \Pr(z_3\notin S \wedge z'=z_3) = \Pr(z'=z_3) \cdot \Pr(z'\notin S \mid z'=z_3).
\end{align*}
and,
\begin{align*}
     \Pr(z'=z_3)= \frac{1}{n}.
\end{align*}
Moreover, 
\begin{align*}
    \Pr(z'\notin S \mid z'=z_3)
    = P(z_3\notin S)
    = (1-\frac{1}{n})^n
    \geq \frac{1}{e}(1-\frac{1}{n})
    \geq \frac{1}{2e}.
\end{align*}
Combining everything together proves our claim.
\end{proof}

\begin{lemma} \label{A_2_nonlip}
Assume $n\geq 35$ and let $\prob_2$  be the fraction of $z_2$ in $S$. Then
\begin{align*}
    \Pr(A_2 \mid A_1)
    =\Pr\brk!{ \delta_2\in [ \tfrac{1}{32}, \tfrac{1}{8}] \mid A_1 } \geq \frac{1}{60}.
\end{align*}
\end{lemma}

\begin{proof}
Let $p'_i = \Pr(z_i=z_2 \mid A_1)$.
Note that by the fact that $i\neq j$, $z_i, z_j$ are indepndent.
Then, for every $i\neq j$, $p'_i=p'_j$.
Then, by the fact that every example is independent,
\begin{align*}
    p'_i
    &=\Pr(z_i=z_2 \mid z_3\notin S)
    \\
    &=\Pr(z_i=z_2 \mid z_i\neq z_3)
    \\
    &=\frac{\Pr(z_i=z_2)}{\Pr(z_i\neq z_3)}
    \\
&=\frac{1}{1-\frac{1}{n}} \Pr(z_i=z_2)
\\&=\frac{5}{64}.
\end{align*}
Then,
\begin{align*}  \E[\delta_2 \mid A_1]=\frac{1}{n}\sum_{i=1}^n\Pr(z_i=z_2 \mid A_1)=
    \frac{1}{n}\sum_{i=1}^np'_i=\frac{5}{64}
    ,
\end{align*}
and,
\begin{align*}
    \operatorname{Var}(\delta_2 \mid A_1)
    = \operatorname{Var}\brk3{ \frac{1}{n}\sum_{i=1}^n 1_{\{z_i=z_2\}} \mid A_1 }
    =\frac{1}{n^2}\sum_{i=1}^n \operatorname{Var}(1_{\{z_i=z_2\}} \mid A_1)
    =\frac{5\cdot 59}{64^2n} 
    .
\end{align*}
Finally, by Chebyshev's inequality, for $n\geq 35$,
\begin{align*}
    \Pr(A_2 \mid A_1)
    &=\Pr\brk!{ \delta_2\in [ \tfrac{1}{32}, \tfrac{1}{8}] \mid A_1}
    \\
    &=\Pr\brk!{ \abs!{ \delta_2 - \tfrac{5}{64}} \leq \tfrac{3}{64} \mid A_1}
    \\
    &= 1-\Pr\brk!{ \abs!{\delta_2 - \tfrac{5}{64} } \geq \tfrac{3}{64} \mid A_1}
    \\
    &\geq 1-\frac{64^2}{9} \operatorname{Var}(\delta_2 \mid A_1)
    \\
    &= 1-\frac{5\cdot 59}{9n}
    \\
    &\geq 1-\frac{5\cdot 59}{315}
    \\
    &\geq \frac{1}{60}
    .
    \end{align*}
\end{proof}
\subsection{Proof of \cref{lower_smallt}}
\label{app_lower_smallt}

\begin{lemma} \label{smallt_char}
Let $\classfunc$ be a tail function.
Let $\loss(x)$ be the following function,
\[\loss(x)=
\begin{cases}
      \classfunc(x) & \text{if $x\geq 0$;}\\
     \classfunc(0)+\classfunc'(0)x & \text{if $x<0$.}
\end{cases}
\]
Then, $\loss\in \classtail$.
\end{lemma}

\begin{proof}
First, it is easy to verify that $\loss$ is continuously differentiable.
Second, $\loss$ is non negative: for $x\geq 0$ by the non negativity of $\classfunc$ and for $x< 0$ by the fact that $\classfunc'(0)\leq 0$.
Moreover, $\loss$ is convex. We need to prove that every $x<y$, $\loss'(x)\leq \loss'(y)$ For $x,y<0$, we get it by the convexity of $\classfunc$.  For $x,y>0$, we get it by the linearity of $\loss$. For $x<0<y$, by the convexity of $\classfunc$,
\begin{align*}
    \loss'(x)= \classfunc'(0)\leq \classfunc'(y)=\loss'(y).
\end{align*}
In addition, $\loss$ is $\beta$-smooth. We need to prove that every $x<y$, $\loss'(y)-\loss'(x)\leq \beta(y-x)$ For $x,y\geq0$, we get it by the smoothness of $\classfunc$.  For $x,y\leq0$, we get it by the linearity of $\loss$. For $x\leq 0\leq y$, by the smoothness of $\classfunc$,
\begin{align*}
    \loss'(y)-\loss'(x)=\classfunc'(y)-\classfunc'(0)\leq \beta y\leq \beta (y-x).
\end{align*}
Finally, $\loss$ is strictly monotonically decreasing.
We need to prove that every $x<y$, $\loss(y)>\loss(x)$. For $x,y>0$, we get it by the monotonicity of $\classfunc$.  For $x<y<0$, 
\begin{align*}
    \loss(y)=\classfunc(0)+\classfunc'(0)y\leq \classfunc(0)+\classfunc'(0)x=\loss(x).
\end{align*}
For $x<0<y$, 
\begin{align*}
    \loss(y)=\classfunc(y)\leq \classfunc(0) \leq \classfunc(0)+\classfunc'(0)x=\loss(x)
    .
\end{align*}
\end{proof}

\end{document}